\documentclass[preprint,12pt]{elsarticle}

\usepackage{amssymb}
\usepackage{graphicx}
\usepackage{amsmath}
\usepackage{caption}
\usepackage{subcaption}
\usepackage{bm}
\usepackage{amsfonts}
\usepackage{algorithm}
\usepackage{algpseudocode}
\usepackage{amsthm}
\usepackage{tablefootnote}
\usepackage{float}
\usepackage{hyperref}

\journal{Expert Systems with Applications}
\biboptions{authoryear}

\newtheorem{theorem}{Theorem}

\begin{document}

\begin{frontmatter}

\title{Linear classifier design under heteroscedasticity in Linear Discriminant Analysis}

\author[label1]{Kojo~Sarfo~Gyamfi\corref{cor1}} 
\ead{gyamfik@uni.coventry.ac.uk}

\author[label1]{James~Brusey}
\ead{j.brusey@coventry.ac.uk}

\author[label1]{Andrew~Hunt}
\ead{ab8187@coventry.ac.uk}

\author[label1]{Elena~Gaura}
\ead{csx216@coventry.ac.uk}

\cortext[cor1]{Corresponding author:}

\address[label1]{Faculty of Engineering and Computing, Coventry University, Coventry, CV1 5FB, United Kingdom}

\begin{abstract}
Under normality and homoscedasticity assumptions, Linear Discriminant Analysis (LDA) is known to be optimal in terms of minimising the Bayes error for binary classification. In the heteroscedastic case, LDA is not guaranteed to minimise this error. Assuming heteroscedasticity, we derive a linear classifier, the Gaussian Linear Discriminant (GLD), that directly minimises the Bayes error for binary classification. In addition, we also propose a local neighbourhood search (LNS) algorithm to obtain a more robust classifier if the data is known to have a non-normal distribution. We evaluate the proposed classifiers on two artificial and ten real-world datasets that cut across a wide range of application areas including handwriting recognition, medical diagnosis and remote sensing, and then compare our algorithm against existing LDA approaches and other linear classifiers. The GLD is shown to outperform the original LDA procedure in terms of the classification accuracy under heteroscedasticity. While it compares favourably with other existing heteroscedastic LDA approaches, the GLD requires as much as $60$ times lower training time on some datasets. Our comparison with the support vector machine (SVM) also shows that, the GLD, together with the LNS, requires as much as $150$ times lower training time to achieve an equivalent classification accuracy on some of the datasets. Thus, our algorithms can provide a cheap and reliable option for classification in a lot of expert systems. 
\end{abstract}


\end{frontmatter}

\section{Introduction}
\label{}
In many applications one encounters the need to classify a given object under one of a number of distinct groups or classes based on a set of features known as the feature vector. A typical example is the task of classifying a machine part under one of a number of health states. Other applications that involve classification include face detection, object recognition, medical diagnosis, credit card fraud prediction and machine fault diagnosis.

A common treatment of such classification problems is to model the conditional density functions of the feature vector \citep{ng2002discriminative}. Then, the most likely class to which a feature vector belongs can be chosen as the class that maximises the a posteriori probability of the feature vector. This is known as the maximum a posteriori (MAP) decision rule.

Let $K$ be the number of classes, $\mathcal{C}_k$ be the $k$th class, $\textbf{x}$ be a feature vector and $\mathcal{D}_k$ be training samples belonging to the $k$th class $(k\in \{ 1,2,...,K\} )$. The MAP decision rule for the classification task is then to choose the most likely class of $\textbf{x}$, $\mathcal{C}^*(\textbf{x})$ given as:
\begin{equation}
\mathcal{C}^*(\textbf{x})=\arg \max_{\mathcal{C}_k} p(\mathcal{C}_k|\textbf{x}), \quad k\in \{ 1,2,...,K\}
\end{equation}

We assume for the moment that there are only $K=2$ classes, i.e. binary classification (we consider multi-class classification in a later section). Then, using Bayes' rule, the two posterior probabilities can be expressed as:
\begin{equation}
p(\mathcal{C}_k|\textbf{x})=\frac{p(\textbf{x}|\mathcal{C}_k)\times p(\mathcal{C}_k)}{p(\textbf{x})}, \quad {k\in \{ 1,2\} }
\end{equation} 

It is often the case that the prior probabilities $p(\mathcal{C}_1)$ and $p(\mathcal{C}_2)$ are known, or else they may be estimable from the relative frequencies of $\mathcal{D}_1$ and $\mathcal{D}_2$ in $\mathcal{D}$ where $\mathcal{D}=\mathcal{D}_1 \cup \mathcal{D}_2$. Let these priors be given by $\pi_1$ and $\pi_2$ respectively for class $\mathcal{C}_1$ and $\mathcal{C}_2$. Then, the likelihood ratio defined as:
\begin{equation}
\lambda(\textbf{x})=\frac{p(\textbf{x}|\mathcal{C}_1)}{p(\textbf{x}|\mathcal{C}_2)}
\end{equation}
is compared against a threshold defined as $\tau=\pi_2/\pi_1$ so that one decides on class $\mathcal{C}_1$ if $\lambda(\textbf{x})\geq \tau$ and class $\mathcal{C}_2$ otherwise.

Linear Discriminant Analysis (LDA) proceeds from here with two basic assumptions \citep[Chapter~8]{izenman2009modern}:
\begin{enumerate}
	\item The conditional probabilities $p(\textbf{x}|\mathcal{C}_1)$ and $p(\textbf{x}|\mathcal{C}_2)$ have multivariate normal distributions.
	\item The two classes have equal covariance matrices, an assumption known as homoscedasticity.
\end{enumerate}

Let $\bar{\textbf{x}}_1$, $\bm{\Sigma}_1$ be the mean and covariance matrix of $\mathcal{D}_1$ and $\bar{\textbf{x}}_2$, $\bm{\Sigma}_2$ be the mean and covariance of $\mathcal{D}_2$ respectively. Then, for $k\in \{ 1,2\}$,
\begin{equation}
p(\textbf{x}|\mathcal{C}_k)=\frac{1}{\sqrt{(2\pi)^d\det(\bm{\Sigma}_k)}}\exp\bigg [-\frac{1}{2}(\textbf{x}-\bar{\textbf{x}}_k)^T\bm{\Sigma}_k^{-1}(\textbf{x}-\bar{\textbf{x}}_k) \bigg]
\end{equation}
where $d$ is the dimensionality of $\mathcal{X}$, which is the feature space of $\textbf{x}$.
Given the above definitions of the conditional probabilities, one may obtain a log-likelihood ratio given as:
\begin{align}
&\ln \lambda(\textbf{x})\nonumber\\
&=\frac{1}{2}\ln \frac{\det \bm{\Sigma}_2}{\det \bm{\Sigma}_1}+\frac{1}{2} \bigg [(\textbf{x}-\bar{\textbf{x}}_2)^T\bm{\Sigma}_2^{-1}(\textbf{x}-\bar{\textbf{x}}_2)-(\textbf{x}-\bar{\textbf{x}}_1)^T\bm{\Sigma}_1^{-1}(\textbf{x}-\bar{\textbf{x}}_1) \bigg]
\end{align}	
which is then compared against $\ln \tau$ so that $\mathcal{C}_1$ is chosen if $\ln \lambda(\textbf{x})\geq \ln \tau$, and $\mathcal{C}_2$ otherwise. Thus, the decision rule for classifying a vector $\textbf{x}$ under class $\mathcal{C}_1$ can be rewritten as:
\begin{equation}
(\textbf{x}-\bar{\textbf{x}}_2)^T\bm{\Sigma}_2^{-1}(\textbf{x}-\bar{\textbf{x}}_2)-(\textbf{x}-\bar{\textbf{x}}_1)^T\bm{\Sigma}_1^{-1}(\textbf{x}-\bar{\textbf{x}}_1) \geq  \ln \frac{\tau^2\det \bm{\Sigma}_1}{\det \bm{\Sigma}_2}
\end{equation}
In general, this result is a quadratic discriminant. However, a linear classifier is often desired for the following reasons:	
\begin{enumerate}
	\item A linear classifier is robust against noise since it tends not to overfit \citep{mika1999fisher}.	
	\item A linear classifier has relatively shorter training and testing times \citep{yuan2012recent}.
	\item Many linear classifiers allow for a transformation of the original feature space into a higher dimensional feature space using the kernel trick for better classification in the case of a non-linear decision boundary \citep[Chapter~6]{bishop2006pattern}.
\end{enumerate}

By calling on the assumption of homoscedasticity, i.e. $\bm{\Sigma}_1=\bm{\Sigma}_2=\bm{\Sigma}_x$, the original quadratic discriminant given by (6) for classifying a given vector $\textbf{x}$ decomposes into the following linear decision rule:
\begin{equation}
\textbf{x}^T\bm{\Sigma}_x^{-1}(\bar{\textbf{x}}_1-\bar{\textbf{x}}_2) \mathop{\gtreqless}_{\mathcal{C}_2}^{\mathcal{C}_1} \ln\tau+\frac{1}{2}(\bar{\textbf{x}}_1^T\bm{\Sigma}_x^{-1}\bar{\textbf{x}}_1-\bar{\textbf{x}}_2^T\bm{\Sigma}_x^{-1}\bar{\textbf{x}}_2)
\end{equation}
Here, $\bm{\Sigma}_x^{-1}(\bar{\textbf{x}}_1-\bar{\textbf{x}}_2)$ is a vector of weights denoted by $\textbf{w}$ and $\ln\tau+\frac{1}{2}(\bar{\textbf{x}}_1^T\bm{\Sigma}_x^{-1}\bar{\textbf{x}}_1-\bar{\textbf{x}}_2^T\bm{\Sigma}_x^{-1}\bar{\textbf{x}}_2)$ is a threshold denoted by $w_0$. This linear classifier is also known as Fisher’s Linear Discriminant. If only the weight vector \textbf{w} is required for dimensionality reduction, \textbf{w} may be obtained by maximising Fisher’s criterion \citep{fisher1936use}, given by:
\begin{equation}
S=\dfrac{\textbf{w}^T(\bar{\textbf{x}}_1-\bar{\textbf{x}}_2)(\bar{\textbf{x}}_1-\bar{\textbf{x}}_2)^T\textbf{w}}{\textbf{w}^T\bm{\Sigma}_x\textbf{w}}
\end{equation}
where $\bm{\Sigma}_{x}=n_1\bm{\Sigma}_1+n_2\bm{\Sigma}_2$ and $n_1$, $n_2$ are the cardinalities of $\mathcal{D}_1$ and $\mathcal{D}_2$ respectively.

LDA is the optimal Bayes' classifier for binary classification if the normality and homoscedasticity assumptions hold \citep{hamsici2008bayes} \citep[Chapter~8]{izenman2009modern}. It demands only the computation of the dot product between $\textbf{w}$ and $\textbf{x}$, which is a relatively computationally inexpensive operation.

As a supervised learning algorithm, LDA is performed either for dimensionality reduction (usually followed by classification) \citep[Chapter~16]{barber2012bayesian}, \citep{buturovic1994toward,duin2004linear,sengur2008expert}, or directly for the purpose of statistical classification \citep[Chapter~4]{fukunaga2013introduction}\citep{izenman2009modern,mika1999fisher}. LDA has been applied to several problems such as medical diagnosis e.g. \citep{sharma2008cancer,coomans1978application,sengur2008expert,polat2008cascade}, face and object recognition e.g. \citep{song2007parameterized,chen2000new,liu2007efficient,yu2001direct} and credit card fraud prediction e.g. \citep{mahmoudi2015detecting}. The widespread use of LDA in these areas is not because the datasets necessarily satisfy the normality and homoscedasticity assumptions, but mainly due to the robustness of LDA against noise, being a linear model \citep{mika1999fisher}. Since the linear Support Vector Machine (SVM) can be quite expensive to train, especially for large values of $K$ or $n$ ($n=n_1+n_2$), LDA is often relied upon \citep{hariharan2012discriminative}. 

Yet, practical implementation of LDA is not without problems. Of note is the small sample size (SSS) problem that LDA faces with high-dimensional data and much smaller training data \citep{sharma2015linear,lu2003regularized}. When $d\gg n$, the scatter matrix $\Sigma_x$ is not invertible, as it is not full-rank. Since the decision rule as given by (7) requires the computation of the inverse of $\Sigma_x$, the singularity of $\Sigma_x$ makes the solution infeasible. In works by, for example, \citep{liu2007efficient,paliwal2012improved}, this problem is overcome by taking the Moore-Penrose pseudo-inverse of the scatter matrix, rather than the ordinary matrix inverse. \cite{sharma2008cancer} use a gradient descent approach where one starts from an initial solution of $\textbf{w}$ and moves in the negative direction of the gradient of Fisher's criterion (8). This method avoids the computation of an inverse altogether. Another approach to solving the SSS problem involves adding a scalar multiple of the identity matrix to the scatter matrix to make the resulting matrix non-singular, a method known as regularised discriminant analysis \citep{friedman1989regularized,lu2003regularized}.

However, for a given dataset that does not satisfy the homoscedasticity or normality assumption, one would expect that modifications to the original LDA procedure accounting for these violations would yield an improved performance. One such modification, in the case of a non-normal distribution, is the mixture discriminant analysis \citep{hastie1996discriminant,mclachlan2004discriminant,ju2003gaussian} in which a non-normal distribution is modelled as a mixture of Gaussians. However, the parameters of the mixture components or even the number of mixture components, are usually not known a priori. Other non-parametric approaches to LDA that remove the normality assumption involve using local neighbourhood structures \citep{cai2007locality,fukunaga1983nonparametric,li2009nonparametric} to construct a similarity matrix instead of the scatter matrix $\Sigma_x$ used in LDA. However, these approaches aim at linear dimensionality reduction, rather than linear classification. Another modification, in the case of a non-linear decision boundary between $\mathcal{D}_1$ and $\mathcal{D}_2$, is the Kernel Fisher Discriminant (KFD) \citep{mika1999fisher,zhao2009multiclass,polat2008cascade}. KFD maps the original feature space $\mathcal{X}$ into some other space $\mathcal{Y}$ (usually higher dimensional) via the kernel trick \citep{mika1999fisher}. While the main utility of the kernel is to guarantee linear separability in the transformed space, the kernel may also be employed to transform non-normal data into one that is near-normal.

Our proposed method differs from the above approaches in that we primarily consider violation of the homoscedasticity assumption, and do not address the SSS problem. We seek to provide a linear approximation to the quadratic boundary given by (6) under heteroscedasticity without any kernel transformation; we note that several heteroscedastic LDA approaches have been proposed to this effect. Nevertheless, for reasons which we highlight in the next section, our contributions in this paper are stated explicitly as follows:

\begin{enumerate}
	\item We propose a novel linear classifier, which we term the Gaussian Linear Discriminant (GLD), that directly minimises the Bayes error under heteroscedasticity via an efficient optimisation procedure. This is presented in Section 3.
	\item We propose a local neighbourhood search method to provide a more robust classifier if the data has a non-normal distribution (Section 4).
\end{enumerate}

\section{Related Work}
Under the heteroscedasticity assumption, many LDA approaches have been proposed among which we mention \citep[Chapter~4]{fukunaga2013introduction},\citep{zhang2008equalized,mclachlan2004discriminant,duin2004linear,decell1976feature,decell1977feature,malina1981extended,loog2002non}. As it is known that Fisher's criterion (whose maximisation is equivalent to the LDA derivation described in the Introduction section) only takes into account the difference in the projected class means, existing heteroscedastic LDA approaches tend to obtain a generalisation on Fisher's criterion. In the work of \citep{loog2002non}, for instance, a directed distance matrix (DDM) known as the Chernoff distance, which takes into account the difference in covariance matrices between the two classes as well as the projected class means, is maximised instead of Fisher's criterion (8). The same idea employing the Chernoff criterion is used by \citep{duin2004linear}. A wider class of Bregman divergences including the Bhattacharya distance \citep{decell1976feature} and the Kullback-Leibler divergence \citep{decell1977feature} have also been used for heteroscedastic LDA, as Fisher's criterion can be considered a special case of these measures when the covariance matrices of the classes are equal.

However, most of these approaches aim at linear dimensionality reduction, which involves finding a linear transformation that transforms the original data into one of reduced dimensionality, while at the same time maximising the discriminatory information between the classes. Our focus with this paper, however, is not on dimensionality reduction, but on obtaining a Bayes optimal linear classifier for binary classification assuming that the covariance matrices are not equal. As far as we know, the closest work to ours in this regard are the works by \citep{marks1974discriminant,anderson1962classification,peterson1966method,fukunaga2013introduction}

Obtaining the Bayes optimal linear classifier involves minimising the probability of misclassification $p_e$ as given by:
\begin{equation}
p_e=\pi_1p(y<w_0|\mathcal{C}_1)+\pi_2p(y \geq w_0|\mathcal{C}_2)
\end{equation}
where $y=\textbf{w}^T\textbf{x}$.
Unfortunately, there is no closed-form solution to the minimisation of (9) \citep{anderson1962classification}. Thus, an iterative procedure is inevitable in order to obtain the Bayes optimal linear classifier.

In the work of \citep{marks1974discriminant}, for example, the iterative procedure described is to solve for $\textbf{w}$ and $w_0$ as given by
\begin{align}
&\textbf{w}=\big[s_1\bm{\Sigma}_1+s_2\bm{\Sigma}_2\big]^{-1}(\bar{\textbf{x}}_1-\bar{\textbf{x}}_2)\nonumber\\
& w_0=\mu_1-s_1\sigma_1^2=\mu_2+s_2\sigma_2^2
\end{align}
by obtaining the optimal values of $s_1$ and $s_2$ via systematic trial and error. We denote this heteroscedastic LDA procedure by R-HLD-2, for the reason that the two parameters $s_1$ and $s_2$ are chosen at random.

\citep{anderson1962classification} makes the observation that if the weight vector $\textbf{w}$ and the threshold $w_0$ are both multiplied by the same positive scalar, the decision boundary remains unchanged. Therefore, by multiplying (10) through by the scalar $s_1+s_2$, $\textbf{w}$ and $w_0$ can be put in the form of:
\begin{align}
&\textbf{w}=\big[s\bm{\Sigma}_2+(1-s)\bm{\Sigma}_1\big]^{-1}(\bar{\textbf{x}}_1-\bar{\textbf{x}}_2)\nonumber\\
& w_0=\mu_1-(1-s)\sigma_1^2=\mu_2+s\sigma_2^2
\end{align}
Still, the optimal value of $s$ has to be chosen by systematic trial and error. We denote this heteroscedastic LDA approach by R-HLD-1, for the reason that only one parameter $s$ is chosen at random. As we show in the next section, $s$ is unbounded. Therefore, the difficulty faced by this approach is that $s$ has to be chosen from the interval $(-\infty,\infty)$, so that the probability of finding the optimal $s$ for a given dataset is low, without extensive trial and error to limit the choice of $s$ to some finite interval $[a,b]$. 

To avoid the unguided trial and error procedure in \citep{marks1974discriminant,anderson1962classification}, \citep{peterson1966method} and \citep[Chapter~4]{fukunaga2013introduction} propose a theoretical approach described below:
\begin{enumerate}
	\item Change $s$ from $0$ to $1$ with small step increments $\Delta s$.
	
	\item Evaluate $\textbf{w}$ as given by:	
	\begin{equation}
	\textbf{w}=\big[s\bm{\Sigma}_1+(1-s)\bm{\Sigma}_2\big]^{-1}(\bar{\textbf{x}}_1-\bar{\textbf{x}}_2)
	\end{equation}
	
	\item Evaluate $w_0$ as given by:
	\begin{equation}
	w_0=\frac{s\mu_2\sigma_1^2+(1-s)\mu_1\sigma_2^2}{s\sigma_1^2+(1-s)\sigma_2^2}
	\end{equation}
	
	\item Compute the probability of misclassification $p_e$.
	
	\item Choose $\textbf{w}$ and $w_0$ that minimise $p_e$. 	 
	
\end{enumerate}
We refer to this procedure as C-HLD, for the reason that the optimal $s$ is constrained in the interval $[0,1]$.

However, we highlight two main problems with the above C-HLD procedure:
\begin{enumerate}
	\item There is no obvious choice of the step rate $\Delta s$. Too small a value of $\Delta s$ will demand too many matrix inversions in Step 2, as there will be too many $s$ values. On the other hand, if $\Delta s$ is too large, the optimal $s$ may not be refined enough, and the $\textbf{w}$ obtained may not be optimal. Specifically, the change in $\textbf{w}$ that results from a small change in $s$ is given as:
   \begin{equation}
	\mathrm{d}\textbf{w}=\big(s\bm{\Sigma}_2+(1-s)\bm{\Sigma}_1\big)^{-1}(\bm{\Sigma}_1-\bm{\Sigma}_2)\big(s\bm{\Sigma}_2+(1-s)\bm{\Sigma_1}\big)^{-1}(\bar{\textbf{x}}_1-\bar{\textbf{x}}_2)\mathrm{d}s
   \end{equation}
   which can affect the classification accuracy.
	
    \item The solution obtained this way is only locally optimal as $s$ is bounded in the interval $[0,1]$. As we show in the next section, $s$ is actually unbounded. When there is a class imbalance \citep{xue2008unbalanced}, the optimal $s$ may be found outside the interval $[0,1]$ which can lead to poor classification accuracy. 
\end{enumerate}
Our proposed algorithm, which is described in the next section, unlike the trial and error approach by \citep{marks1974discriminant,anderson1962classification}, has a principled optimisation procedure, and unlike \citep{fukunaga2013introduction,peterson1966method} does not encounter the problem of choosing an inappropriate $\Delta s$, nor restricts $s$ to the interval $[0,1]$. Consequently, our proposed algorithm achieves a far lower training time than the C-HLD, R-HLD-1 and R-HLD-2, for roughly the same classification accuracy.

\section{Gaussian Linear Discriminant}
Let $\textbf{w}\in \mathbb{R}^d$ be a vector of weights, and $w_0\in\mathbb{R}$, a threshold such that:
\begin{equation}
\mathcal{C}^*(\textbf{x})=
\begin{cases}
\mathcal{C}_1 & \quad \text{if} \quad y=\textbf{w}^T\textbf{x}\geq w_0\\
\mathcal{C}_2 & \quad \text{if} \quad y=\textbf{w}^T\textbf{x}<w_0 \\
\end{cases}
\end{equation}
Since $\textbf{x}$ is assumed to have a multivariate normal distribution in classes $\mathcal{C}_1$ and $\mathcal{C}_2$, $y$ has a mean of $\mu_1$ and a variance of $\sigma_1^2$ for class $\mathcal{C}_1$ and a mean of $\mu_2$ and a variance of $\sigma_2^2$ for class $\mathcal{C}_2$ given as:	
\begin{equation}
\mu_1=\textbf{w}^T\bar{\textbf{x}}_1 \quad \mu_2=\textbf{w}^T\bar{\textbf{x}}_2 \quad \sigma_1^2=\textbf{w}^T\bm{\Sigma}_1\textbf{w} \quad \sigma_2^2=\textbf{w}^T\bm{\Sigma}_2\textbf{w}
\end{equation}

With reference to the Bayes error of (9), the individual misclassification probabilities can be expressed as:
\begin{align}
& p(y<w_0|\mathcal{C}_1)\nonumber \\
&=\int_{-\infty}^{w_0}\frac{1}{\sqrt{2\pi}\sigma_1}\exp \bigg[-\frac{(\zeta-\mu_1)^2}{2\sigma_1^2} \bigg]d\zeta=1-Q\bigg(\frac{w_0-\mu_1}{\sigma_1}\bigg)
\end{align}
and
\begin{equation}
p(y\geq w_0|\mathcal{C}_2)=\int_{w_0}^{\infty}\frac{1}{\sqrt{2\pi}\sigma_2}\exp \bigg[-\frac{(\zeta-\mu_2)^2}{2\sigma_2^2} \bigg]d\zeta=Q\bigg(\frac{w_0-\mu_2}{\sigma_2}\bigg)
\end{equation}
where $Q(^.)$ is the Q-function.
Therefore, the Bayes error to be minimised may be rewritten as:
\begin{equation}
p_e=\pi_1\big[1-Q(z_1)\big]+\pi_2\big[Q(z_2)\big]
\end{equation}
where
\begin{equation}
z_1=\frac{w_0-\mu_1}{\sigma_1} \quad \text{and} \quad z_2=\frac{w_0-\mu_2}{\sigma_2}
\end{equation} 

Our aim is to find a local minimum of $p_e$. A necessary condition is for the gradient of $p_e$ to be zero, i.e.,
\begin{equation}
\nabla p_e(\textbf{w},w_0)=\bigg[\frac{\partial p_e}{\partial \textbf{w}^T},\frac{\partial p_e}{\partial w_0} \bigg]^T=\textbf{0}
\end{equation}
From (9), it can be shown that:
\begin{equation}
\frac{\partial p_e}{\partial \textbf{w}}=\pi_1 \bigg(\frac{1}{\sqrt{2\pi}}e^{-z_1^2/2}\frac{\partial z_1}{\partial\textbf{w}} \bigg) -\pi_2\bigg(\frac{1}{\sqrt{2\pi}}e^{-z_2^2/2}\frac{\partial z_2}{\partial \textbf{w}}  \bigg)
\end{equation}
From (20), however, we obtain the following:
\begin{equation}
\frac{\partial z_1}{\partial\textbf{w}}=\frac{-\sigma_1\bar{\textbf{x}}_1-z_1\bm{\Sigma}_1\textbf{w}}{\sigma_1^2} \quad \text{and} \quad \frac{\partial z_2}{\partial\textbf{w}}=\frac{-\sigma_2\bar{\textbf{x}}_2-z_2\bm{\Sigma}_2\textbf{w}}{\sigma_2^2}
\end{equation}
Therefore,
\begin{equation}
\frac{\partial p_e}{\partial \textbf{w}}=\frac{1}{\sqrt{2\pi}}\bigg[ -\pi_1e^{-z_1^2/2} \bigg(\frac{\sigma_1\bar{\textbf{x}}_1+z_1\bm{\Sigma}_1\textbf{w}}{\sigma_1^2}\bigg) + \pi_2e^{-z_2^2/2} \bigg(\frac{\sigma_2\bar{\textbf{x}}_2+z_2\bm{\Sigma}_2\textbf{w}}{\sigma_2^2}\bigg)\bigg]
\end{equation}
It can similarly be shown from (9) that, 
\begin{equation}
\frac{\partial p_e}{\partial w_0}=\pi_1 \bigg(\frac{1}{\sqrt{2\pi}}e^{-z_1^2/2}\frac{\partial z_1}{\partial w_0} \bigg)-\pi_2\bigg(\frac{1}{\sqrt{2\pi}}e^{-z_2^2/2}\frac{\partial z_2}{\partial w_0}  \bigg)
\end{equation}
Again, from (20),
\begin{equation}
\frac{\partial z_1}{\partial w_0}=\frac{1}{\sigma_1} \quad \text{and} \quad \frac{\partial z_2}{\partial w_0}=\frac{1}{\sigma_2}
\end{equation}
Therefore,
\begin{equation}
\frac{\partial p_e}{\partial w_0}=\frac{\pi_1}{\sqrt{2\pi}} \bigg(\frac{1}{\sigma_1}e^{-z_1^2/2}\bigg)-\frac{\pi_2}{\sqrt{2\pi}}\bigg(\frac{1}{\sigma_2}e^{-z_2^2/2}  \bigg)
\end{equation}

Now, equating the gradient $\nabla p_e(\textbf{w},w_0)$ to zero, the following set of equations are obtained:
\begin{equation}
\bigg(\frac{\pi_2z_2}{\sigma_2^2}e^{-z_2^2/2}\bm{\Sigma}_2-\frac{\pi_1z_1}{\sigma_1^2}e^{-z_1^2/2}\bm{\Sigma}_1\bigg)\textbf{w}=\bigg(\frac{\pi_1}{\sigma_1}e^{-z_1^2/2} \bigg)\bar{\textbf{x}}_1-\bigg(\frac{\pi_2}{\sigma_2}e^{-z_2^2/2}\bigg)\bar{\textbf{x}}_2
\end{equation}
\begin{equation}
\frac{\pi_1}{\sigma_1}e^{-z_1^2/2}=\frac{\pi_2}{\sigma_2}e^{-z_2^2/2}
\end{equation}
Substituting (29) into (28) yields:
\begin{equation}
\bigg(\frac{z_2}{\sigma_2}\bm{\Sigma}_2-\frac{z_1}{\sigma_1}\bm{\Sigma}_1\bigg)\textbf{w}=(\bar{\textbf{x}}_1-\bar{\textbf{x}}_2)
\end{equation}
Then the vector $\textbf{w}$ can be given by:
\begin{equation}
\textbf{w}=\bigg(\frac{z_2}{\sigma_2}\bm{\Sigma}_2-\frac{z_1}{\sigma_1}\bm{\Sigma}_1\bigg)^{-1}(\bar{\textbf{x}}_1-\bar{\textbf{x}}_2)
\end{equation}

It will be noted however that (31) is still in terms of $w_0$, so that an explicit representation of $w_0$ in terms of $\textbf{w}$ is needed from (29) to substitute in $z_1$ and $z_2$ in (31). This is where our approach most significantly differs from \citep{fukunaga2013introduction}. Solving for $w_0$ from (29) results in the following quadratic:
\begin{equation}
\frac{z_2^2}{2}-\frac{z_1^2}{2}-\ln\bigg(\frac{\tau\sigma_1}{\sigma_2}\bigg)=0
\end{equation}
which can be simplified to:
\begin{equation}
\bigg(\frac{w_0-\mu_2}{\sigma_2}\bigg)^2-\bigg(\frac{w_0-\mu_1}{\sigma_1}\bigg)^2 -2\ln \frac{\tau\sigma_1}{\sigma_2}=0,
\end{equation}
where $\tau$ is given as before as $\tau=\pi_2/\pi_1$. If $\tau$ is defined and not equal to zero, and $\sigma_1^2\not=\sigma_2^2$ (since $\bm{\Sigma}_1\not=\bm{\Sigma}_2$ for heteroscedastic LDA), (33) can be shown to have the following solutions:
\begin{equation}
w_0=\frac{\mu_2\sigma_1^2-\mu_1\sigma_2^2\pm\sigma_1\sigma_2\sqrt{(\mu_1-\mu_2)^2+2(\sigma_1^2-\sigma_2^2)\ln\big(\frac{\tau\sigma_1}{\sigma_2}\big)}}{\sigma_1^2-\sigma_2^2}
\end{equation}

Nevertheless, since there are two solutions to $w_0$ in (34), a choice has to be made as to which of them is substituted into (31). To eliminate one of the solutions, we consider the second-order partial derivative of $p_e$ with respect to $w_0$ evaluated at $w_0$ as given by (34), and determine under what condition it is greater than or equal to zero. This is a second-order necessary condition for $p_e$ to be a local minimum. From (27), it can be shown that:
\begin{equation}
\frac{\partial^2 p_e}{\partial w_0^2}=\frac{\pi_1}{\sqrt{2\pi}} \bigg(-\frac{z_1}{\sigma_1^2}e^{-z_1^2/2}\bigg)+\frac{\pi_2}{\sqrt{2\pi}}\bigg(\frac{z_2}{\sigma_2^2}e^{-z_2^2/2}  \bigg)
\end{equation}
We denote this second-order derivative by $h$. We then consider all possibilities of $z_1$ and $z_2$ (which are the variables in (35) that depend on $w_0$) under three cases, and analyse the sign of $h$ in each.

\subsubsection*{Case 1}
$z_2\leq 0$ and $z_1\geq 0$: then $h$ is trivially non-positive.	
\subsubsection*{Case 2}
$z_2\geq 0$ and $z_1\leq 0$: then $h$ is trivially non-negative.	
\subsubsection*{Case 3}
$z_2>0$ and $z_1>0$ or $z_2<0$ and $z_1<0$: then $h$ is non-negative if and only if
\begin{equation}
\ln\bigg(\frac{\pi_2z_2}{\sigma_2^2}\bigg)-\frac{z_2^2}{2} \geq \ln\bigg(\frac{\pi_1z_1}{\sigma_1^2}\bigg)-\frac{z_1^2}{2}
\end{equation}
i.e.,
\begin{equation}
\ln\bigg(\frac{z_2}{\sigma_2}\bigg/ \frac{z_1}{\sigma_1}\bigg) \geq \frac{z_2^2}{2}-\frac{z_1^2}{2}-\ln\bigg(\frac{\tau\sigma_1}{\sigma_2}\bigg)
\end{equation}
It will be noted that the right-hand side of the inequality (37) is identically zero, as can be seen from (32). Therefore, the condition under which $h$ is greater than or equal to zero is when:
\begin{equation}
\frac{z_2}{\sigma_2} \geq \frac{z_1}{\sigma_1}
\end{equation}	
Note also that Case 2 necessarily satisfies (38) so that we consider (38) as the general inequality for the non-negativity of $h$ for all cases, and thus for $w_0$ to be a local minimum.

Now, when one considers the two solutions of $w_0$ in (35), only the solution given by:
\begin{equation}
w_0=\frac{\mu_2\sigma_1^2-\mu_1\sigma_2^2+\sigma_1\sigma_2\sqrt{(\mu_1-\mu_2)^2+2(\sigma_1^2-\sigma_2^2)\ln\big(\frac{\tau\sigma_1}{\sigma_2}\big)}}{\sigma_1^2-\sigma_2^2}
\end{equation}
satisfies the inequality of (38), i.e., only this choice of $w_0$ corresponds to a local minimum. The proof of this is given in the appendix. 

We may then substitute this expression of $w_0$ into (31) so that (31) is in terms of $\textbf{w}$ only. Even so, $\textbf{w}$ has to be solved for iteratively. This is because (31) has no closed-form solution since $\mu_1,\mu_2,\sigma_1,\sigma_2$ are themselves functions of $\textbf{w}$. As the iterative procedure requires an initial choice of $\textbf{w}$, we use Fisher's choice of the weight vector as given by:
\begin{equation}
\textbf{w}=(n_1\bm{\Sigma}_1+n_2\bm{\Sigma}_2)^{-1}(\bar{\textbf{x}}_1-\bar{\textbf{x}}_2)
\end{equation}
as our initial solution. Again, we mention that $n_1$ and $n_2$ are the cardinalities of $\mathcal{D}_1$ and $\mathcal{D}_2$.  After a number of such iterative updates, the optimal $w_0$ is then solved for from (39).
This algorithm, known as the Gaussian Linear Discriminant (GLD), is described in detail in Algorithm 1.

\begin{algorithm}
	\caption{GLD}\label{alg1}
	\begin{algorithmic}[1]
		\State Input: $\mathcal{D}_1$ and $\mathcal{D}_2$
		\State Evaluate $\bar{\textbf{x}}_1,\bar{\textbf{x}}_2,\bm{\Sigma}_1,\bm{\Sigma}_2$
		\State Initialise $\textbf{w}$: $\textbf{w}=(n_1\bm{\Sigma}_1+n_2\bm{\Sigma}_2)^{-1}(\bar{\textbf{x}}_1-\bar{\textbf{x}}_2)$
		\State Evaluate $\mu_1,\mu_2,\sigma_1^2,\sigma_2^2,z_1,z_2$.
		\While {Stopping criteria are not satisfied}
		\State Solve for $w_0$ from (39)
		\State Evaluate $z_1,z_2$
		\State Evaluate the Bayes error $p_e$
		\State Update $\textbf{w}$ as $\textbf{w}=\bigg(\frac{z_2}{\sigma_2}\bm{\Sigma}_2-\frac{z_1}{\sigma_1}\bm{\Sigma}_1\bigg)^{-1}(\bar{\textbf{x}}_1-\bar{\textbf{x}}_2)$
		\State Evaluate $\mu_1,\mu_2,\sigma_1,\sigma_2$.		
		\EndWhile		
	\end{algorithmic}
\end{algorithm}
Note that by multiplying both $\textbf{w}$ of (31) and $w_0$ proportionally by $c=(\sigma_1z_2-\sigma_2z_1)/\sigma_1\sigma_2$ (due to (38), $c$ is non-negative and hence the discrimination criterion given by (15) is not changed), the GLD may be viewed in terms of the optimal solution of (11), where
\begin{equation}
	s=-\sigma_2z_1/(\sigma_1z_2-\sigma_2z_1).
\end{equation} which is unbounded given the inequality of (38).
However, unlike \citep{anderson1962classification,marks1974discriminant}, $s$ is not chosen by systematic trial and error, and unlike \citep{fukunaga2013introduction}, $s$ is not varied between $0$ and $1$ at small step increments. Instead, since $s$ is a function of $\textbf{w}$ and $w_0$, our algorithm may be interpreted as obtaining increasingly refined values of $s$ by improving upon $\textbf{w}$ and $w_0$ starting from Fisher's solution, as is described in Algorithm 1.

\subsection{Stopping Criteria}
The GLD algorithm may be terminated under any of the following conditions:
\begin{enumerate}
	\item When the change in the objective function $p_e$ remains within a certain tolerance $\epsilon_1$ for a number of consecutive iterations.
	\item When the change in the norm of $\textbf{w}$ remains within a certain tolerance $\epsilon_2$ for a number of consecutive iterations.
    \item When the gradient of $p_e$ as given by (21) remains within a certain tolerance $\epsilon_3$ for a number of consecutive iterations.
	\item After a fixed number of iterations $I$, if convergence is slow.
\end{enumerate}
At the end of the algorithm, the final solution may be chosen either as the solution to which the iterations converge, or the solution corresponding to the minimum $p_e$ found in the iterative updates.

\subsection{Multiclass Classification}
Suppose now that there are $K>2$ classes in the dataset $\mathcal{D}$, then the classification problem may be reduced to a number of binary classification problems. The two main approaches usually taken for this reduction are the One-vs-All (OvA) and One-vs-One (OvO) strategies \citep{bishop2006pattern,hsu2002comparison}.

\subsubsection{One-vs-All (OvA)} In OvA, one trains a classifier to discriminate between one class and all other classes. Thus, there are $K$ different classifiers. An unknown vector $\textbf{x}$ is then tested on all $K$ classifiers so that the class corresponding to the classifier with the highest discriminant score is chosen. However, with respect to the proposed GLD algorithm, this is an ill-suited approach. This is because the collection of all other classes on one side of the discriminant will not necessarily have a normal distribution, and could in fact be multimodal, if the means are well-separated. Since our algorithm is built on strong normality assumptions of the data on each side of the discriminant, the GLD, as has been formulated, is expected to perform poorly.

\subsubsection{One-vs-One}
In OvO, a classifier is trained to discriminate between every pair of classes in the dataset, ignoring the other $K-2$ classes. Thus, there are $K(K-1)/2$ unique classifiers that may be constructed. Again, an unknown vector $\textbf{x}$ is tested on all $K(K-1)/2$ classifiers. The predicted classes for all the classifiers are then tallied so that the class that occurs most frequently is chosen. This is equivalent to a majority vote decision. In a lot of cases, however, there is no clear-cut winner, as more than one class may have the highest number of votes. In such a case, the most likely class is often chosen randomly between those most frequently occurring classes. The GLD provides a more appropriate means for breaking such ties, by making use of the minimised Bayes error $p_e$ for each classifier. Specifically, one may instead use a weighted voting system, where the count of every predicted class is weighted by $1-p_e$, since $p_e$ provides an appropriate measure of uncertainty associated with each classifier output. Thus, the decision rule is reduced to choosing the maximum weighted vote among the $K$ classes.

Note that even though the GLD minimises the Bayes error for each classifier, the overall Bayes error for a multiclass problem may not be minimised by using multiple binary classifiers.

\section{Non-normal Distributions}
So far, the fundamental assumption that has been used to derive the GLD is that the data in each class has a normal distribution. Thus, for an unknown non-normal distribution, the linear classifier we have obtained does not minimise the Bayes error for that unknown distribution. We argue, however, that if this unknown distribution is nearly-normal \citep{mudholkar2000epsilon}, then a more robust linear classifier may be found in some neighbourhood of the GLD. For this reason, we use a local neighbourhood search algorithm to explore the region in $\mathbb{R}^{d+1}$ around the GLD to obtain the classifier that minimises the number of misclassifications on the training dataset. We do this by perturbing each of the $d+1$ vector elements in the optimal $\tilde{\textbf{w}}=[w_0,\textbf{w}^T]^T$ obtained from the GLD procedure by a small amount $\delta\tilde{w}_i$. After every perturbation, the resulting classifier is evaluated on the test dataset. This procedure is repeated as described in Algorithm 2 until the stopping criterion is satisfied.

\begin{algorithm}[!tbph]
	\caption{Local Neighbourhood Search (LNS)}\label{alg2}
	\begin{algorithmic}[1]
		\State Input: Optimal $\tilde{\textbf{w}}=[w_0,\textbf{w}^T]^T$ obtained from the GLD.
		\While {Stopping criterion is not satisfied}
		\State Let $\tilde{\textbf{w}}$ be the current solution.
		\For {$i\gets 1$ to $d$}
		\State $\textbf{v}^+\gets\tilde{\textbf{w}}$, $\textbf{v}^-\gets\tilde{\textbf{w}}$.
		\State $\textbf{v}^+\gets v_i^++\delta v_i^+$
		\State Evaluate the misclassifications on the training set using $\textbf{v}^+$
		\State $\textbf{v}^-\gets v_i^--\delta v_i^-$
		\State Evaluate the misclassifications on the training set using $\textbf{v}^-$
		\EndFor
		\State Set the classifier with the minimum number of misclassifications as the current solution $\tilde{\textbf{w}}$.
		\EndWhile
		\State Choose the classifier with the smallest number of misclassifications.	
	\end{algorithmic}
\end{algorithm}
The algorithm is terminated after a certain maximum number of iterations $R$ is reached. Additionally, one may perform an early termination if after a predefined number of iterations $r_{max}$, there is no improvement in the minimum number of misclassifications on the training dataset that has been found in the search.

\section{Experimental Validation}
We validate our proposed algorithm on two artificial datasets denoted by D1 and D2, as well as on ten real-world datasets taken from the University of California, Irvine (UCI) Machine Learning Repository. These datasets are shown in Table 1, and cut across a wide range of applications including handwriting recognition, medical diagnosis, remote sensing and spam filtering. D1 and D2 are normally distributed with different covariance matrices. For D1, we generate $1000$ samples for class $\mathcal{C}_1$ and $2000$ samples for class $\mathcal{C}_2$ using the following Gaussian parameters:
\begin{align}
& \bar{\textbf{x}}_2=[3.86,3.10,0.84,0.84,1.64,1.08,0.26,0.01]^T, \nonumber\\ &\bm{\Sigma}_2=\text{diag}(8.41,12.06,0.12,0.22,1.49,1.77,0.35,2.73) \nonumber\\
& \bar{\textbf{x}}_1=\bar{\textbf{x}}_2-0.3, \quad \bm{\Sigma}_1=\textbf{I}
\end{align}
For D2, we generate $2000$ samples for class $\mathcal{C}_1$ and $4000$ samples for class $\mathcal{C}_2$ using the following Gaussian parameters:
\begin{align}
& \bar{\textbf{x}}_2=[-1.5,-0.75,0.75,1.5]^T, \nonumber\\ 
&\bm{\Sigma}_2=\text{diag}(0.25,0.75,1.25,1.75) \nonumber\\
& \bar{\textbf{x}}_1=\bar{\textbf{x}}_2-0.75, \quad \bm{\Sigma}_1=\textbf{I}
\end{align}
The above Gaussian parameters are slightly modified from the two class data used by \citep{fukunaga2013introduction} and \citep{xue2008unbalanced} in order to make the sample means less separated.
	
	\begin{table}[!tbph]
		\centering
		\caption{List and Characteristics of Datasets}
		{\begin{tabular}{ccccc}
				\hline
				Dataset & Label & $n$ & $d$ & $K$ \\
				\hline
				
				D1 & (a) & $ 2000 $ & $ 8 $ & $ 2 $ \\
				
				D2 & (b) & $ 2000 $ & $ 4 $ & $ 2 $ \\
				
				Liver & (c) & $345$ & $6$ & $ 2 $ \\
				
				Shuttle & (d) & $ 58000 $ & $ 9 $ & $7$ \\
				
				Vowels & (e) & $ 990 $ & $ 10 $ & $ 11 $ \\
				
				Zernike Moments & (f) & $2000$ & $47$ & $10$ \\
				
				Image Segmentation (Statlog) & (g) & $2310$ & $19$ & $7$ \\
				
				Spambase & (h) & $ 4601 $ & $ 37 $ & $ 2 $ \\
				
				Wine Quality (White) & (i) &$4898$ & $11$ & $7$ \\
				
				Pen Digits & (j) & $ 5620 $ & $ 64 $ & $ 10 $ \\
				
				Satellite (Statlog) & (k) & $ 6435 $ & $ 36 $ & $ 6 $ \\
				
				Letters & (l) & $ 20000 $ & $ 16 $ & $ 26 $ \\				
				
				\hline
			\end{tabular}}{}
			\caption*{This table lists the datasets used in the experimental section. $K$ is the number of classes, $d$ is the dimensionality of the dataset, and $n$ is the number of data points in the dataset.}		
			\label{DS1}
		\end{table}
	
	For each dataset in Table 1, we perform $10$-fold cross validation. We run $20$ different trials. On each training dataset, we evaluate the minimum Bayes error achievable by our proposed algorithm averaged over all $10$ folds and $20$ trials. If there are more than two classes, we use OvO, and calculate the mean Bayes error over all $K(K-1)/2$ discriminants. As we are interested only in linear classification, we compare the performance of the GLD with the original LDA as well as the heteroscedastic LDA procedures by \citep{fukunaga2013introduction},\citep{anderson1962classification} and \citep{marks1974discriminant} as described in Section 2 in terms of the Bayes error (9). For the sake of brevity, we denote these three heteroscedastic LDA algorithms by the annotations earlier introduced: C-HLD, R-HLD-1 and R-HLD-2 respectively. These results are shown in Table 2.
	
	Moreover, for each of the test datasets, we evaluate the average classification accuracy for each of LDA, C-HLD, R-HLD-1, R-HLD-2, GLD and GLD with local neighbourhood search (LNS). We also compare the performance of these LDA approaches to the SVM. These results are shown in Table 3, while the average training times of the algorithms are shown in Table 4.  
	
	We estimate the prior probabilities based on the relative frequencies of the data in each class in the dataset, and the stopping criterion for the GLD is thus: we stop if the gradient of $\textbf{w}$ change is less than or equal to $\epsilon_3=10^{-6}$, or else we terminate our algorithm after $I=20$ iterations and choose the solution corresponding to the minimum $p_e$. Also, for the LNS procedure, we perturb each vector element by $10\%$ of its absolute value, i.e. $\delta\tilde{w}_i=0.1|\tilde{w}_i|$, and we run for $R$=1000 iterations, terminating prematurely if $r_{max}=0.1R$. We use a step size of $\Delta s=0.001$ for the C-HLD algorithm, and run $1000$ trials for R-HLD-1 and R-HLD-2. All the parameters used in the experiments are optimised via cross-validation. Note that if the sample covariance matrix is singular, we use the Moore-Penrose pseudo-inverse.		
	
	\section{Results and Discussion}
	\begin{table}[!tbph]
					\centering
					\setlength{\tabcolsep}{12pt}
					\caption{Average Bayes Error ($\%$)}
					\resizebox{\columnwidth}{!}{
						{\begin{tabular}{cccccc}
								\hline
								Dataset & LDA & C-HLD & R-HLD-1 & R-HLD-2 & GLD \\
								\hline
								
								(a) & $ 0.0397$ & $ 0.0382$ & $ 0.0383$ & $ 0.0361$ & $ \textbf{0.0360}$ \\
								
								(b) & $ 0.0774$ & $ 0.0749$ & $ 0.0749$ & $ 0.0740$ & $ \textbf{0.0739}$ \\
								
								(c) & $ 0.9981$ & $ \textbf{0.9838}$ & $ \textbf{0.9838}$ & $ \textbf{0.9838}$ & $ \textbf{0.9838}$ \\
								
								(d) & $ \textbf{0.0001}$ & $ \textbf{0.0001}$ & $ \textbf{0.0001}$ & $ \textbf{0.0001}$ & $ \textbf{0.0001}$ \\
								
								(e) & $ 0.0339$ & $ \textbf{0.0326}$ & $ \textbf{0.0326}$ & $ \textbf{0.0326}$ & $ \textbf{0.0326}$ \\
								
								(f) & $ 0.0054$ & $ 0.0051$ & $ \textbf{0.0048}$ & $ \textbf{0.0048}$ & $ 0.0050$ \\
								
								(g) & $ 0.0037$ & $ \textbf{0.0029}$ & $ \textbf{0.0029}$ & $ \textbf{0.0029}$ & $ \textbf{0.0029}$ \\
								
								(h) & $ 0.0253$ & $ \textbf{0.022}8$ & $ \textbf{0.0228}$ & $ \textbf{0.0228}$ & $ \textbf{0.0228}$ \\
								
								(i) & $ 0.0162$ & $ 0.0201$ & $ 0.0156$ & $ 0.0155$ & $ \textbf{0.0154}$ \\
								
								(j) & $ \textbf{0.0002}$ & $ \textbf{0.0002}$ & $\textbf{ 0.0002}$ & $ \textbf{0.0002}$ & $ \textbf{0.0002}$ \\
								
								(k) & $ 0.0046$ & $ \textbf{0.0039}$ & $ \textbf{0.0039}$ & $ \textbf{0.0039}$ & $ \textbf{0.0039}$ \\
								
								(l) & $ \textbf{0.0007}$ & $ \textbf{0.0007}$ & $ \textbf{0.0007}$ & $ \textbf{0.0007}$ & $ \textbf{0.0007}$ \\
			
								\hline
							\end{tabular}}{}
							
						}
						\caption*{This table shows the average Bayes error per discriminant as a percentage for each dataset for LDA, GLD, C-HLD, R-HLD-1 and R-HLD-2. Best values are in bold.}
						\label{Bayes}
					\end{table}

		\begin{table}[!tbph]
			\centering
			\caption{Average Classification Accuracy ($\%$)}
			\resizebox{\columnwidth}{!}{
				{\begin{tabular}{cccccccc}
						\hline
						Dataset & LDA & C-HLD & R-HLD-1 & R-HLD-2 & GLD & LNS & SVM\\
						\hline
						
						(a) & $ 76.00$ & $ 77.18$ & $ 77.00$ & $ 78.48$ & $ \textbf{78.65}$ & $ 78.57$ & $ 77.47$\\
						
						(b) & $ 76.87$ & $77.93$ & $ 77.93$ & $ 78.17$ & $ \textbf{78.37}$ & $ 78.00$ & $ 77.70$\\
						
						(c) & $ 67.83$ & $ 63.19$ & $ 62.32$ & $ 62.03$ & $ 63.77$ & $ 68.12$ & $ \textbf{68.70}$\\
						
						(d) & $ 94.10$ & $ 96.60$ & $ 96.74$ & $ 96.73$ & $ 96.59$ & $ \textbf{97.91}$ & $ 84.39$\\
						
						(e) & $ 73.64$ & $ 74.14$ & $ 74.44$ & $ 74.44$ & $ 74.14$ & $75.66$ & $ \textbf{76.77}$\\
						
						(f) & $ 84.00$ & $ 83.90$ & $ 84.10$ & $ 84.15$ & $ \textbf{84.80}$ & $ 84.00$ & $ 81.90$\\
						
						(g) & $ 94.33$ & $ 94.59$ & $ 94.59$ & $ 94.63$ & $ 94.59$ & $ 94.89$ & $ \textbf{96.15}$\\
						
						(h) & $ 88.76$ & $ 88.29$ & $ 88.26$ & $ 88.15$ & $ 88.26$ & $ \textbf{90.28}$ & $ 85.68$\\
						
						(i) & $ 53.41$ & $ 46.59$ & $ 53.37$ & $ 53.33$ & $ 53.55$ & $ \textbf{54.1}4$ & $ 51.88$\\
						
						(j) & $ 96.74$ & $ 96.99$ & $ 96.97$ & $ 96.98$ & $ 97.01$ & $ 97.41$ & $ \textbf{97.84}$\\
						
						(k) & $ 85.69$ & $ 86.06$ & $ 86.06$ & $ 86.03$ & $ 86.08$ & $ 86.65$ & $ \textbf{86.85}$\\
						
						(l) & $ 81.67$ & $ 81.87$ & $ 81.83$ & $ 81.78$ & $ 81.88$ & $ 82.25$ & $ \textbf{85.39}$\\
						
						\hline
					\end{tabular}}{}
					
				}
				\caption*{This table shows the average classification accuracy (\%) on the test datasets for LDA, C-HLD, R-HLD-1, R-HLD-2, GLD, GLD+LNS and SVM. Best values are in bold.}
				\label{Rate}
			\end{table}
			
			\begin{table}[!tbph]
				\centering
				\caption{Average Training Time (s)}
				\resizebox{\columnwidth}{!}{
					{\begin{tabular}{cccccccc}
							\hline
							Dataset & LDA & C-HLD & R-HLD-1 & R-HLD-2 & GLD & LNS & SVM\\
							\hline
					
							(a) & $ \textbf{0.001}$ & $ 0.161$ & $ 0.140$ & $ 0.139$ & $ 0.002$ & $0.181$ & $ 23.192$\\
							
							(b) & $\textbf{0.001}$ & $0.142$ & $0.121$ & $0.121$ & $0.002$ & $ 0.060$ & $ 0.721$\\
							
							(c) & $ \textbf{0.001}$ & $ 0.155$ & $ 0.1415$ & $ 0.1337$ & $ 0.003$ & $ 0.028$ & $ 2.673$\\
							
							(d) & $\textbf{0.037}$ & $ 3.531$ & $ 3.023$ & $ 3.012$ & $0.089$ & $ 43.32$ & $4623.138$\\
							
							(e) & $\textbf{0.036}$ & $ 11.099$ & $ 9.409$ & $ 9.751$ & $ 0.167$ & $ 2.075$ & $ 1.173$\\
							
							(f) & $ \textbf{0.387}$ & $ 123.662$ & $ 123.649$ & $ 121.906$ & $ 1.955$ & $ 110.694$ & $23.126$\\
							
							(g) & $\textbf{0.128}$ & $ 37.320$ & $ 30.876$ & $ 37.875$ & $ 0.488$ & $ 2.143$ & $21.775$\\
							
							(h) & $ \textbf{0.101}$ & $ 10.437$ & $ 7.729$ & $ 7.474$ & $ 0.753$ & $ 36.83$ & $804.574$\\
							
							(i) & $ \textbf{0.017}$ & $ 4.257$ & $ 3.691$ & $ 3.750$ & $ 0.080$ & $ 5.928$ & $ 914.257$\\
							
							(j) & $\textbf{0.638}$ & $ 10.099$ & $9.358$ & $ 9.171$ & $ 0.915$ & $ 168.19$ & $ 409.38$\\
							
							(k) & $ \textbf{0.304}$ & $ 18.067$ & $17.842$ & $ 17.912$ & $ 0.858$ & $ 13.919$ & $311.202$\\
							
							(l) & $\textbf{0.835}$ & $73.050$ & $64.022$ & $65.414$ & $3.245$ & $ 37.202$ & $109.232$\\
							
							\hline
						\end{tabular}}{}
						
					}
					\caption*{This table shows the average training times on the test datasets for LDA, C-HLD, R-HLD-1, R-HLD-2, GLD, GLD+LNS and SVM. Best values are in bold.}
					\label{Time}
				\end{table}
				
				For real-world datasets, the covariance matrices of the classes are rarely equal, therefore the homoscedasticity assumption in LDA does not hold. Our results in Table 2 confirm that LDA does not minimise the Bayes error under heteroscedasticity, as none of the datasets used has equal covariance matrices. With the exception of datasets (d), (j) and (l), where LDA achieves an equal Bayes error as the other heteroscedastic LDA approaches, LDA is outperformed by the GLD on all remaining datasets in terms of minimising the Bayes error. It will be noted that the other three heteroscedastic LDA approaches algorithms achieve a performance comparable to the GLD on all the datasets in terms of the Bayes error. However, R-HLD-1 and R-HLD-2 require a lot of trials ($1000$ in our experiments) in order to obtain the optimal parameters $s$ and $s_1$, $s_2$ respectively, while C-HLD requires a step size of $\Delta s=0.001$ which translates to $1001$ trials. Consequently, the training time for these algorithms far exceed that of the GLD, as can be seen in Table 4. For example, the gain in training time of the GLD over C-HLD, R-HLD-1 and R-HLD-2 is over $62$ folds for dataset (g), and about $20$ folds for dataset (l). Moreover, since C-HLD, R-HLD-1 and R-HLD-2 all require matrix inversions, performing a matrix inversion for each of the $1000$ trials can be a computationally demanding task especially for high-dimensional data, which have large covariance matrices. Instead, since the GLD follows a principled optimisation procedure, the number of matrix inversions required is far lower. For example, on dataset (f), which has a dimensionality of $47$, the GLD requires over $60$ times less time to train than the other heteroscedastic LDA approaches.

				It is conceivable that the minimisation of the Bayes error would translate into a good performance in terms of the classification accuracy, if the normality assumption of LDA holds. For this reason, it can be seen in Table 3 that the GLD achieves the best classification accuracy on datasets (a) and (b), which are generated from known normal distributions. Thus, the proposed GLD algorithm is particularly suited for applications with datasets that tend to be normally distributed in each class e.g. in machine fault diagnosis, or accelerometer-based human activity recognition \citep{ojetola2015data}, as it also requires far less training time than the existing heteroscedastic LDA approaches.
				
				However, for datasets (c) through to (l), the classes do not have any known normal distribution. Therefore, minimising the Bayes error under the normality assumption would not necessarily result in a classifier that has the best classification accuracy, even if the difference in covariance matrices has been accounted for. For this reason, it is not surprising that LDA achieves a superior classification accuracy than C-HLD, R-HLD-1, R-HLD-2 and the GLD on datasets (c) and (h) as can be seen in Table 3. However, by searching around the neighbourhood of the GLD, the Local Neighbourhood Search (LNS) algorithm is able to account for the non-normality and obtain a more robust classifier. Thus, the GLD, together with the LNS procedure, achieves a higher classification accuracy than all the LDA approaches on all the real-world datasets (i.e. (c) to (l)) with the exception of dataset (f) which has the GLD showing superior classification accuracy.
				
			While the SVM outperforms the LDA approaches on half of the datasets, its training time can be rather long for large datasets. For instance, for dataset (d) which has $58000$ elements, the SVM takes about $1.3$ hours to train whereas the GLD with LNS, which achieves the best classification accuracy on this dataset, takes $43$ seconds to train, representing over $100$ fold savings in computational time over the SVM. Similar patterns can be seen in other datasets like (i), where the GLD with LNS achieves a superior classification accuracy with over $150$ times shorter training time than the SVM. This suggests that for such large datasets, the GLD with Local Neighbourhood Search is a low-complexity alternative to the SVM, as it requires far less computational time than the SVM.

            We, however, make note of two weaknesses our proposed algorithms have. For the GLD, the procedure as described in Algorithm 1, may converge to a saddle point, instead of a local minimum. Even if it were to converge to a local minimum, there is no guarantee that is the global optimum solution due to the fact that the objective function $p_e$ is known to be non-convex \cite{anderson1962classification}. Also, since the Local Neighbourhood Search involves evaluating the misclassification rate on the training set for every perturbation, the procedure does not scale well with large amounts of training data. Because of this, it is important to have a good initial solution like the GLD, so that an early termination may be performed if there is no improvement after some number of iterations.
			
			\section{Conclusion}
			In this paper, we have presented the Gaussian Linear Discriminant (GLD), a novel and computationally efficient method for obtaining a linear discriminant for heteroscedastic Linear Discriminant Analysis (LDA) for the purpose of binary classification. Our algorithm minimises the Bayes error via an iterative optimisation procedure that uses Fisher's Linear Discriminant as the initial solution. Moreover, the GLD does not require any parameter adjustments. We have also proposed a local neighbourhood search method by which a more robust linear classifier may be obtained for non-normal distributions. 
			Our experimental results on two artificial and ten real world applications show that when the covariance matrices of the classes are unequal, LDA is unable to minimise the Bayes error. Thus, under heteroscedasticity, our proposed algorithm achieves superior classification accuracy to the LDA for normally distributed classes. While the proposed GLD algorithm compares favourably with other heteroscedastic LDA approaches, the GLD requires a far less training time. Moreover, the GLD, together with the LNS, has been shown to be particularly robust, comparing favourably with the SVM, but requiring far less training time on our datasets. Thus, for expert systems like machine fault diagnosis or human activity monitoring that require linear classification, the proposed algorithms provide a low-complexity, high-accuracy solution.

            While this work has focused on linear classification, on-going work is focused on modifying the GLD procedure for the purpose of linear dimensionality reduction. Moreover, it is of particular interest to us to be able to derive the Bayes error for some known non-normal distributions. An alternative to this is to be able to obtain a kernel that implicitly transforms some data of a known non-normal distribution into a feature space where the classes are normally distributed. Finally, like all local search algorithms, the performance and complexity of the LNS procedure depends on the choice of the initial solution. Therefore, further work that explores the use initial solutions (including the heteroscedastic LDA approaches discussed) other than the GLD for the LNS procedure is being done.

\appendix
\section{}
\begin{theorem}
	Let $w_0^+$ and $w_0^-$ be the two distinct solutions of (34), then $w_0^+$ and $w_0^-$ cannot both satisfy (38) given that $\sigma_1\neq \sigma_2$.
\end{theorem}
\begin{proof}
	Let	
	\begin{equation}
	\beta=\sqrt{(\mu_1-\mu_2)^2+2(\sigma_1^2-\sigma_2^2)\ln\bigg(\frac{\tau\sigma_1}{\sigma_2}\bigg)}
	\end{equation}
	
	and let
	\begin{equation}
	w_0^+=\frac{\mu_2\sigma_1^2-\mu_1\sigma_2^2+\sigma_1\sigma_2\beta}{\sigma_1^2-\sigma_2^2}
	\end{equation}
	
	Then
	\begin{equation}
	\frac{z_2}{\sigma_2}=\frac{(\mu_2-\mu_1)\sigma_2+\beta\sigma_1}{\sigma_2(\sigma_1^2-\sigma_2^2)}, \quad \frac{z_1}{\sigma_1}=\frac{(\mu_2-\mu_1)\sigma_1+\beta\sigma_2}{\sigma_1(\sigma_1^2-\sigma_2^2)}
	\end{equation}
	
	Suppose that $w_0^+$ satisfies (38), then
	\begin{equation}
	\frac{(\mu_2-\mu_1)\sigma_2+\beta\sigma_1}{\sigma_2(\sigma_1^2-\sigma_2^2)}\geq\frac{(\mu_2-\mu_1)\sigma_1+\beta\sigma_2}{\sigma_1(\sigma_1^2-\sigma_2^2)}
	\end{equation}
	i.e.,
	\begin{equation}
	\frac{\beta\sigma_1^2}{\sigma_1^2-\sigma_2^2}\geq\frac{\beta\sigma_2^2}{\sigma_1^2-\sigma_2^2}
	\end{equation}
	This implies that $\sigma_1^2/(\sigma_1^2-\sigma_2^2) > \sigma_2^2/(\sigma_1^2-\sigma_2^2)$ since $\beta$ is a positive scalar.
	
	Consider now $w_0^-$ given as:
	\begin{equation}
	w_0^-=\frac{\mu_2\sigma_1^2-\mu_1\sigma_2^2-\sigma_1\sigma_2\beta}{\sigma_1^2-\sigma_2^2}
	\end{equation}
	
	Then
	\begin{equation}
	\frac{z_2}{\sigma_2}=\frac{(\mu_2-\mu_1)\sigma_2-\beta\sigma_1}{\sigma_2(\sigma_1^2-\sigma_2^2)}, \quad \frac{z_1}{\sigma_1}=\frac{(\mu_2-\mu_1)\sigma_1-\beta\sigma_2}{\sigma_1(\sigma_1^2-\sigma_2^2)}
	\end{equation}
	
	In order for (38) to be satisfied, it can be shown, similar to (A.5), that
	\begin{equation}
	\frac{-\beta\sigma_1^2}{\sigma_1^2-\sigma_2^2}\geq\frac{-\beta\sigma_2^2}{\sigma_1^2-\sigma_2^2}
	\end{equation}
	which can be simplified to give $1\leq0$. Since this conclusion is false, only $w_0^+$ satisfies (26).
\end{proof}

\section*{References}

\bibliography{GLD_Bib}

\end{document}